\documentclass[lettersize,journal]{IEEEtran}
\usepackage{amsmath,amsfonts}
\usepackage{algorithmic}
\usepackage{algorithm}
\usepackage{array}
\usepackage[caption=false,font=normalsize,labelfont=sf,textfont=sf]{subfig}
\usepackage{textcomp}
\usepackage{stfloats}
\usepackage{url}
\usepackage{verbatim}
\usepackage{graphicx}
\usepackage{cite}

\usepackage{tabularx}
\usepackage{hyperref} 
\usepackage{booktabs}
\usepackage{multirow}
\usepackage{enumitem}
\usepackage{amssymb}

\usepackage[table]{xcolor}
\usepackage{makecell}
\usepackage{amsthm}

\newtheorem{theorem}{Theorem}[section]

\newtheorem{lemma}[theorem]{Lemma}

\begin{document}

\title{Progressively Exploring and Exploiting Inference Data to Break Fine-Grained Classification Barrier}

\author{Li-Jun Zhao, Si-Yuan Zhang, Zhen-Duo Chen, Xin Luo, Xin-Shun Xu
\thanks{L.-J. Zhao, S.-Y. Zhang, Z.-D. Chen, X. Luo, and X.-S. Xu are with the School of Software, Shandong University, Jinan, 250101, China. (e-mail: lj\_zhao1028@163.com; 1638682606@qq.com; chenzd.sdu@gmail.com; luoxin.lxin@gmail.com; xuxinshun@sdu.edu.cn).}
\thanks{This paper is published in Chinese in SCIENTIA SINICA Informationis. The final published version is available at \url{https://doi.org/10.1360/SSI-2025-0502}.}
}


\maketitle

\begin{abstract}
Current fine-grained classification research primarily focuses on fine-grained feature learning. However, in real-world scenarios, fine-grained data annotation is challenging, and the features and semantics are highly diverse and frequently changing. These issues create inherent barriers between traditional experimental settings and real-world applications, limiting the effectiveness of conventional fine-grained classification methods. Although some recent studies have provided potential solutions to these issues, most of them still rely on limited supervised information and thus fail to offer effective solutions.
In this paper, based on theoretical analysis, we propose a novel learning paradigm to break the barriers in fine-grained classification. This paradigm enables the model to progressively learn during inference, thereby leveraging cost-free data at inference time to more accurately represent fine-grained categories and adapt to dynamic semantic changes.
On this basis, an efficient EXPloring and EXPloiting strategy and method (EXP2) is designed. Thereinto, useful inference data samples are explored according to class representations and exploited to optimize classifiers. Experimental results demonstrate the general effectiveness of our method, providing guidance for future in-depth understanding and exploration of real-world fine-grained classification.
\end{abstract}


\section{Introduction}
\IEEEPARstart{F}{ine-grained} classification \cite{DBLP:conf/iccv/LinRM15, DBLP:journals/tmm/HuangLXWL17} focuses on distinguishing subcategories within a general superclass (e.g., animal species, car models, and aircraft types), which is crucial for many real-world applications. 
For a long time, mainstream research on fine-grained classification has focused on developing better models to enhance fine-grained feature learning capabilities. 
However, in real-world applications, the bigger issue often lies in the data. 
On one hand, the high specialization of fine-grained data categories significantly makes annotation more difficult, thereby greatly limiting the scale of labeled training data; on the other hand, fine-grained tasks, primarily targeting applications in open environments such as ecological protection and security monitoring \cite{DBLP:journals/tmm/LiuZDXLL25}, face challenges of high feature diversity and frequent semantic changes in data.
These issues create inherent barriers between traditional experimental settings and real-world scenarios, limiting the effectiveness of conventional fine-grained classification methods in practical applications.
\begin{figure}[t]
\centering
\includegraphics[width=1\columnwidth]{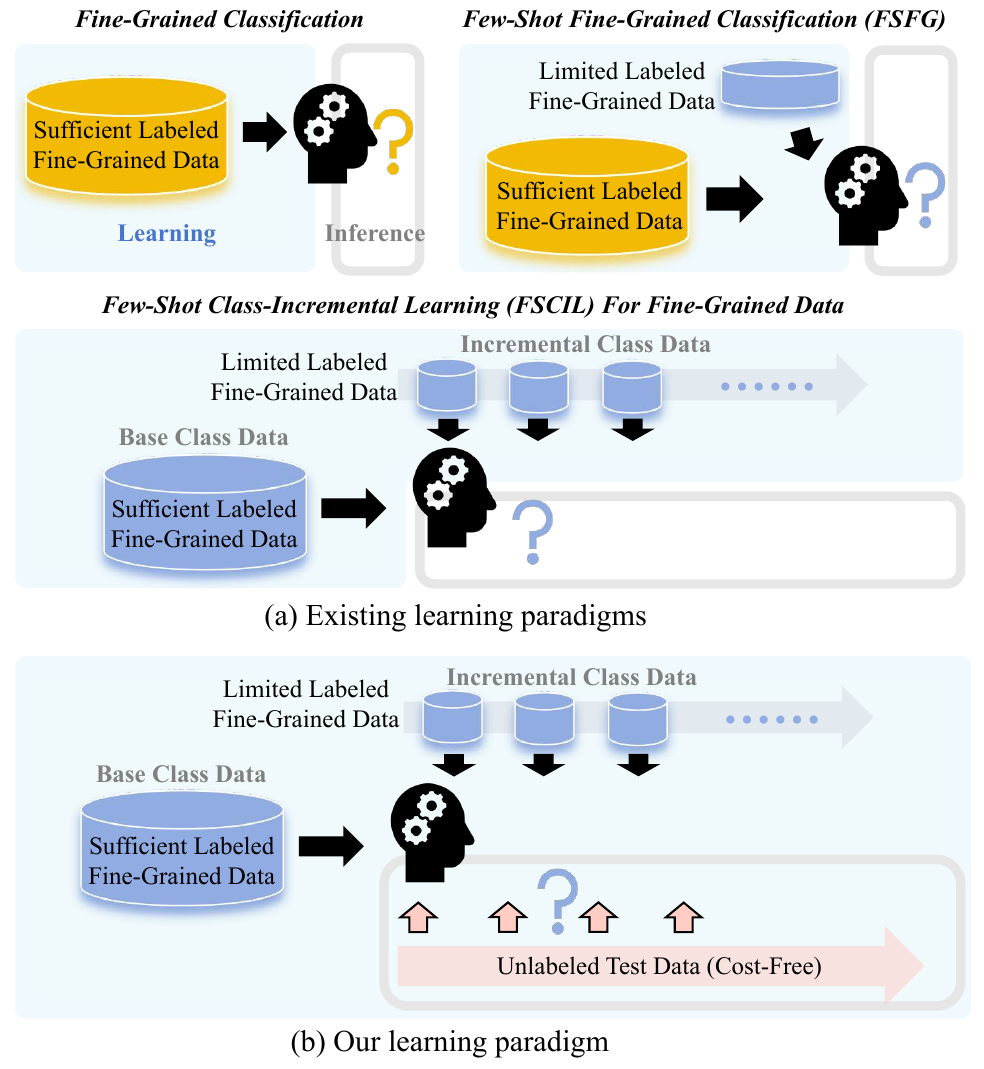} 
\caption{
The diagrams of traditional fine-grained classification, FSFG, and more flexible FSCIL paradigms, along with our proposed learning paradigm of learning during inference.}
\label{fig:idea}
\end{figure}

If considered separately, 
the aforementioned issues pertain to few-shot learning \cite{DBLP:conf/nips/SnellSZ17} and class-incremental learning \cite{DBLP:conf/cvpr/KangPH22}, which aim to reduce supervision information requirements and adapt models to new semantics that emerge in open environments.
As shown in Fig. \ref{fig:idea}(a), 
unlike traditional fine-grained classification, few-shot fine-grained classification (FSFG) \cite{DBLP:journals/tip/Wei0LSW19} aims to enable models to recognize novel fine-grained classes using only a few labeled samples.
This is typically achieved after the models have been initialized or pre-trained on auxiliary datasets with sufficient labeled samples.
Furthermore, few-shot class-incremental learning (FSCIL) \cite{DBLP:conf/cvpr/TaoHCDWG20} endeavors to utilize the abundant training samples of base classes (initial classes) to help learn the continuously arriving incremental classes (subsequent classes) with limited labeled samples. 

\IEEEpubidadjcol
However, 
although the FSCIL paradigm aligns with the requirements of real-world fine-grained classification scenarios in form, it still faces two key limitations. 
1) There is an inherent contradiction between FSCIL and 
accurate fine-grained class representation.
Fine-grained data typically exhibit small inter-class differences and large intra-class variations. 
This means that relying solely on limited supervised data is insufficient for capturing the diverse intra-class features needed to accurately represent fine-grained classes, leading to significant accuracy degradation.
Moreover, 
since the position of limited supervised data within the class distribution remains fixed after training, 
this issue cannot be resolved through model optimization or learning strategies; it can only rely on more diverse intra-class data.
In the experiments, FSCIL methods on fine-grained datasets still perform much worse on incremental classes with limited samples compared to base classes.
This indicates that for fine-grained data, few-shot learning is more of a compromise for limited labeled data rather than an effective solution.
2) In practice,
even if a model follows the FSCIL paradigm to handle streaming data, 
selecting and labeling limited training samples in open environments still requires manual control, preventing the model from reflecting real-time data changes and actual user needs.
Clearly, for real-world fine-grained classification problems, which simultaneously face challenges such as discriminative difficulty, intra-class feature diversity, limited supervised data, and dynamic data changes, it is difficult to find effective solutions directly from existing research.
This motivates us to step beyond current experimental assumptions and approach the problem from a more open perspective.

Based on the above discussion and further theoretical analysis, 
we propose a novel learning paradigm to address the fine-grained classification problem in real-world scenarios. This aims to offer a solution to the contradiction between modeling fine-grained class distributions and limited supervised data, enabling effective modeling of both data distribution and dynamic semantic changes. 
 By incorporating real-world dynamics and temporality, our novel learning paradigm allows learning during inference, rather than treating them as separate phases, as illustrated in Fig. \ref{fig:idea}(b). 
 In the long term, model optimization and meeting user needs occur simultaneously, using all accessible data (including labeled data and inference data from users) for optimization.
This paradigm more closely aligns with real-world applications without additional data collection or labeling.
Based on this, 1) the limitations posed by insufficient supervised data in learning fine-grained class distributions are expected to be overcome, enabling more accurate class representations and better classification boundaries. 2) User data can more accurately and timely reflect task requirements and data distribution changes.
The rational use of user data can enhance the model's adaptability and performance in dynamic data scenarios.

To demonstrate the research value of this novel learning paradigm, we preliminarily design an effective and efficient method based on the FSCIL framework, termed EXP2.
It introduces an optimization procedure that progressively EXPlores useful inference samples relying on class representations and EXPloits them to optimize the classifier distribution in the feature space.
For clearness and fairness, we introduce and evaluate the method based on the FSCIL framework, but it is not limited to this framework. In future work, we will extend this new problem and learning paradigm beyond the FSCIL framework, developing more refined and diverse settings.

Our main contributions are summarized as follows:
\begin{itemize}
    \item 
    By comparing current solutions with real-world scenarios and analyzing barriers in fine-grained classification, we propose a novel learning paradigm that leverages cost-free data to more accurately represent fine-grained categories and adapt to dynamic changes.
    \item Based on this learning paradigm, we design a novel strategy that works as an effective and efficient plug-in method (EXP2) built upon the FSCIL framework.
    \item Extensive experiments demonstrate the general effectiveness of the proposed method and strategy, offering valuable insights for further exploration of the fine-grained classification problem and the FSCIL scenario.
\end{itemize}

\section{Related Work}
\subsection{Fine-Grained Classification} 
Fine-grained classification \cite{DBLP:conf/cvpr/ZhangXZLT16,DBLP:journals/tmm/WangFM24} aims to distinguish different subcategories from a traditional superclass.
Compared to general classification tasks, fine-grained classification requires the model to be more adept at learning subtle differences between images of different classes.
Therefore, fine-grained classification methods put more emphasis on feature learning than traditional methods. 
Among them, part-based methods \cite{DBLP:conf/cvpr/ZhangXZLT16,DBLP:journals/tcsv/WangWLCOT23} focus on identifying discriminative regions using attention sub-networks and extracting distinctive features for classification.
Part-free methods \cite{DBLP:conf/iccv/LinRM15,DBLP:journals/pr/XuWWYWYW24} aim to enhance the backbone’s recognition ability by weighting features or using augmented samples as training data.
Unlike them, some recent methods \cite{DBLP:conf/aaai/Jiang0GDHL24,DBLP:journals/tmm/XuZZGFML24} explore the multimodal information to aid classification. However, obtaining supervised multimodal information is still time-consuming and labor-intensive.

\subsection{Few-Shot Fine-Grained Classification}
Fine-grained classification methods often require a large-scale fine-grained data volume, which is less practical in real scenarios because fine-grained data is challenging to label.
Therefore, Few-Shot Fine-Grained Classification (FSFG) \cite{DBLP:journals/tmm/HuangZZXW21} is proposed to address the challenge of recognizing novel fine-grained categories with a small number of labeled samples.
However, these methods \cite{DBLP:journals/pr/TangYLT22,DBLP:conf/aaai/MaCZ00X24,DBLP:journals/tip/ZhaoCMLX24} still focus on designing more robust network structures, more appropriate optimization objectives, or more sophisticated training techniques to achieve better feature learning with closed training samples, ignoring the more open forms of data that exist in reality.

\subsection{Few-Shot Class-Incremental Learning}
Few-Shot Class-Incremental Learning (FSCIL) \cite{DBLP:conf/nips/ZouZLL22,DBLP:journals/tmm/LiCXLB24} defines a more open and realistic dynamic learning scenario. 
After learning base classes with sufficient training samples, the model can continuously learn new classes with limited labeled samples, allowing the model to distinguish all learned classes.
Due to the high proportion of base classes in the total number of classes, the accuracy of incremental classes, which is crucial for the FSCIL problem, is often overlooked  \cite{DBLP:conf/aaai/DongHTCWG21,DBLP:conf/iccv/CheraghianRRFSP21}.
Some recent methods \cite{DBLP:conf/nips/Wang0ZZY23,DBLP:journals/corr/abs-2402-00481} have identified this issue and tried to address it,
but their designs still limit model optimization to the training samples.
Compared to traditional datasets with large inter-class differences, characterizing class distributions using limited data is more challenging for fine-grained datasets.
Although recent work \cite{DBLP:journals/pr/PanZYZG24,zhao2025evolving} has turned attention to the FSCIL problem of fine-grained datasets, the solutions still maintain the rigid design of previous FSCIL methods.

\subsection{Test-Time Adaptation}
\label{sec:ttar}
Test-Time Adaptation (TTA) \cite{DBLP:conf/nips/IwasawaM21,DBLP:conf/iclr/JangCC23} typically focuses on adapting to the target domain using test-time data. 
Continual test-time adaptation has been proposed to handle the continuous changes in test data distributions. However, all these TTA methods do not consider the performance on the source domain or semantic expansion; they focus solely on adapting to the test data.
In contrast, our method aims to leverage test data to refine the fine-grained class representations established by supervised data in dynamic and open scenarios, where new semantics continuously emerge and labeled data remains scarce. 
Moreover, recent TTA efforts focus on large Vision-Language Models \cite{DBLP:conf/cvpr/KarmanovGLEX24,DBLP:conf/aaai/Liu0PZ24}, which differ significantly from the problem we aim to address.

\begin{figure*}[!t]
\centering
\includegraphics[width=1\textwidth]{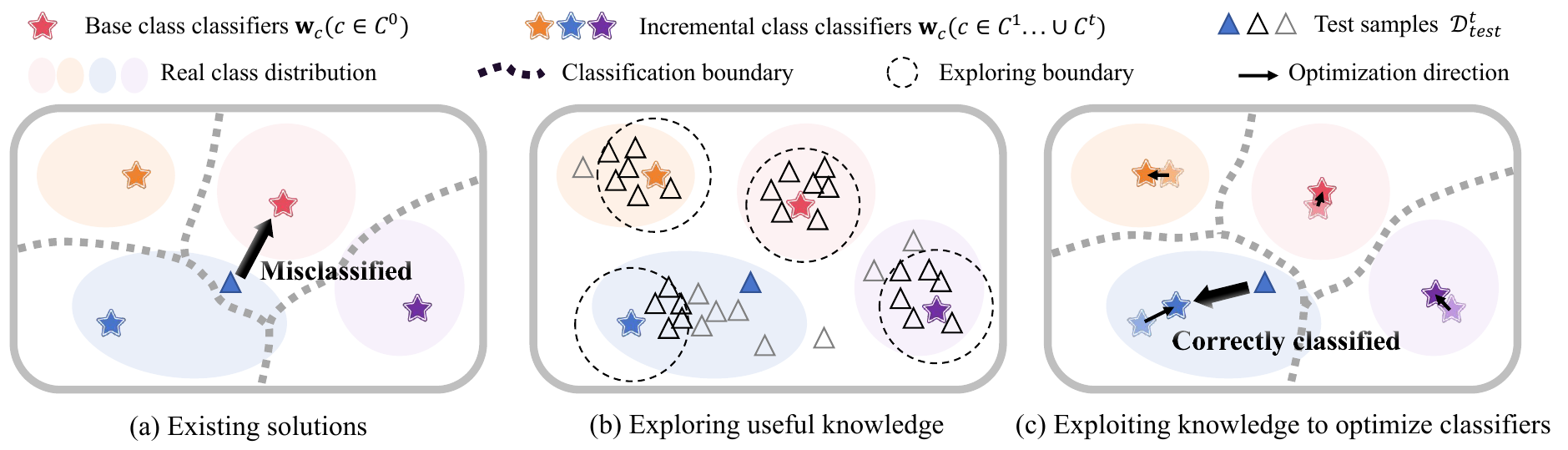} 
\caption{Illustration of existing solutions and EXP2 method. 
(a) Existing solutions use scarce training samples to obtain classifiers, which often causes classifiers to deviate from the real distribution, leading to misclassification. 
(b)(c) EXP2 dynamically explores useful knowledge from test samples to optimize classifiers, thereby obtaining more accurate classification boundaries and enabling correct classification.}
\label{fig1}
\end{figure*}
\section{Methodology}

\subsection{Problem Definition}
To ensure the clearness of the paper and the fairness of the experiments, 
we start from the standard FSCIL setup, the problem our method aims to solve can be defined as follows:

The entire system runtime can be divided into $(T+1)$ sessions, with training sets $\left\{\mathcal{D}_{train}^{t}\right\}_{t=0}^T$ arriving in $(T+1)$ sessions. 
$\mathcal{D}_{train}^{0}$ in the base session (session 0) contains sufficient labeled samples. 
$\mathcal{D}_{train}^{t} (t\ge 1)$ in the incremental session $t$ with scarce labeled samples can be organized in an $N$-way $K$-shot format, i.e., only $K$ samples for each of the $N$ classes.
This setup considers both the scarcity of labeled data and the emergence of new fine-grained semantics in real-world fine-grained classification tasks.
Each $\mathcal{D}_{train}^{t}$ consists of pairs $\left(\mathbf{x}, y\right)$, where $\mathbf{x}$ is a sample from class $y \in \mathcal{C}^{t}$.
$\mathcal{C}^{t}$ is the label set of $\mathcal{D}_{train}^{t}$, and $|\mathcal{C}^{t}|$ represents the number of classes in the set $\mathcal{C}^{t}$. 
All classes $\mathcal{C}^{0}\cdots\cup\mathcal{C}^T$ belong to a superclass.
$\forall i,j\mathrm{~and~}i\neq j,\mathcal{C}^{i}\cap\mathcal{C}^{j}=\varnothing.$

The inter-class and intra-class differences of fine-grained categories are defined as follows,
\begin{equation}
\delta_{\text{inter}} = \min_{c\neq c'}\|\mathbf{\mu}_c - \mathbf{\mu}_{c'}\|,\ 
\sigma_{\text{intra}}^2 = \mathbb{E}_c [\text{Tr}(\mathbf{\Sigma}_c)],
\end{equation}
where $\mathbf{\mu}_c$ and $\mathbf{\Sigma}_c$ denote the class mean and covariance respectively. $\mathbb{E}_c[\cdot]$ denotes the mean over all classes.

Each session is divided into training and inference. 

\subsubsection{Training}  
In each session $t$, only $\mathcal{D}_{train}^{t}$ can be accessed.
The model $\phi(\cdot )$ can be decomposed into feature extractor $f(\cdot )$ and linear classifiers $\mathbf{W}$, i.e., 
$\phi(\mathbf{x}) = \mathbf{W}^{\top}f(\mathbf{x}) \in \mathbb{R}^{\left | \mathcal{C}^0 \cdots\cup\mathcal{C}^t\right |\times 1}$, where $f(\mathbf{x}) \in \mathbb{R}^{d \times 1}$ and $\mathbf{W}=\{\mathbf{w}_c| c\in \mathcal{C}^{0}\cdots\cup\mathcal{C}^t\}\in \mathbb{R}^{d \times  \left | \mathcal{C}^0 \cdots\cup\mathcal{C}^t\right |}$.
The weight of the classifier has been extended to represent the novel classes $\mathcal{C}^t$ introduced by the new session $t$, ensuring that the classifiers $W$ remain aligned with all seen classes.

\subsubsection{Inference} 
The feature extractor $f(\cdot )$ and classifiers $\mathbf{W}$ are used for prediction.
In each session $t$, the model needs to be tested on samples $\{\mathbf{x}|\mathbf{x} \in \mathcal{D}_{test}^{t}\}$ from all seen classes, i.e., $\mathcal{C}^0\cup\mathcal{C}^1\cdots\cup\mathcal{C}^t$.
 The average classification accuracy calculated in this context is referred to as {\bf overall accuracy}. 
 The average classification accuracy calculated specifically for the test samples of classes that arrive in the incremental sessions with insufficient training samples (i.e., $\mathcal{C}^1\cdots\cup\mathcal{C}^t$) is called {\bf incremental class accuracy}.


\subsection{The Basic Strategy of EXP2}
In this section, we provide a theoretical analysis of the inherent barriers in fine-grained classification, and from this analysis, we derive a novel learning paradigm and strategy.

\begin{lemma}
\label{lem:overlap}
Let $\hat{\mu}_c$ be estimated class means with $\max_c\|\hat{\mu}_c - \mu_c\| \leq \epsilon$. The probability of inter-class overlap satisfies:
\begin{equation}
\mathbb{P}(\mathrm{Overlap}) \geq 1 - \Phi\left(\frac{\delta_{\mathrm{inter}} - 2\epsilon}{2\sigma_{\mathrm{intra}}}\right),
\end{equation}
where $\Phi$ is the standard normal cumulative distribution function (CDF) and $\epsilon$ is the estimation error.
\end{lemma}
\begin{proof}
Assuming Gaussian features $X \sim \mathcal{N}(\mathbf{\mu}_c, \sigma_{\text{intra}}^2\mathbf{I})$, consider the decision boundary shift between classes $c$ and $c'$, the condition for being misclassified as class $c'$ is
\begin{align}
\|X - \hat{\mu}_{c'}\|^{2} &< \|X - \hat{\mu}_{c}\|^{2}, \label{eq:misclassification_condition}
\end{align}
simplifying, we get
\begin{align}
(\hat{\mu}_{c}-\hat{\mu}_{c'})^{\top}X &< \frac{\|\hat{\mu}_{c}\|^{2}-\|\hat{\mu}_{c'}\|^{2}}{2}. \label{eq:simplified_inequality1}
\end{align}

Let $\hat{\mu}_{c}=\mu_{c}+\Delta_{c}$, $\hat{\mu}_{c'}=\mu_{c'}+\Delta_{c'}$, where $\|\Delta_{c}\|,\|\Delta_{c'}\|\leq\epsilon$. 
By substituting into Eq. \ref{eq:simplified_inequality1}, 
we obtain the approximate inequality:
\begin{align}
\begin{aligned}
(\mu_{c}-\mu_{c'}+\Delta_{c}-\Delta_{c'})^{\top}X < &\frac{\|\mu_{c}\|^{2}-\|\mu_{c'}\|^{2}}{2}\\
&+(\mu_{c}^{\top}\Delta_{c}-\mu_{c'}^{\top}\Delta_{c'}),
\end{aligned} \label{eq:approximated_inequality}
\end{align}
where second-order terms $\|\Delta_{c}\|^{2}$ and $\|\Delta_{c'}\|^{2}$ are omitted for simplicity.

Introduce the variable $Z$ to represent the projection of the decision boundary without errors:
\begin{align}
Z &=2(\mu_{c'}-\mu_{c})^{\top}X+(\|\mu_{c}\|^{2}-\|\mu_{c'}\|^{2}). \label{eq:define_Z}
\end{align}
The distribution of $Z$ is
\begin{align}
\mathbb{E}[Z] = 
-\|\mu_{c} - \mu_{c'}\|^{2} = -\delta_{\text{inter}}^{2},
\end{align}
\begin{align}
\text{Var}(Z) =4\|\mu_{c} - \mu_{c'}\|^{2}\sigma_{\text{intra}}^{2}&\ge4\delta_{\text{inter}}^{2}\sigma_{\text{intra}}^{2}. \label{eq:variance_Z}
\end{align}
Therefore, $Z\sim N(-\delta_{\text{inter}}^{2},4\delta_{\text{inter}}^{2}\sigma_{\text{intra}}^{2})$.

After introducing the estimation error, the decision boundary becomes
\begin{align}
\begin{aligned}
Z &=2(\mu_{c'}-\mu_{c})^{\top}X+(\|\mu_{c}\|^{2}-\|\mu_{c'}\|^{2}) \\
&>2(\Delta_{c}-\Delta_{c'})^{\top}X-2(\mu_{c}^{\top}\Delta_{c}-\mu_{c'}^{\top}\Delta_{c'})=E.
\end{aligned}\label{eq:define_Z}
\end{align}
We consider the worst-case alignment where $\Delta_{c}$ and $\Delta_{c'}$ lie in the same direction as $(\mu_{c} - \mu_{c'})$, and $\|\Delta_{c}\|=\|\Delta_{c'}\|=\epsilon$.
the random variable $E$ is
\begin{align}
\mathbb{E}[E] = 
-2(\mu_{c}-\mu_{c'})^{\top}\Delta_{c'} \leq -2\epsilon\delta_{\text{inter}},
\end{align}
\begin{align}
\text{Var}(E) =4\sigma_{\text{intra}}^{2}(\Delta_{c}-\Delta_{c'})^{2}&=0, \label{eq:variance_Z}
\end{align}
The inter-class overlap caused by estimation error, i.e., the overlap between the true class $c$ distribution and the estimated class $c'$ distribution, will result in class $c$ samples being misclassified. 
Thus,
\begin{align}
\begin{aligned}
\mathbb{P}(\mathrm{Overlap}) \geq \mathbb{P}(Z>E)  &\ge 1-\Phi\left(\frac{-2\epsilon\delta_{\text{inter}}+\delta_{\text{inter}}^{2}}{2\delta_{\text{inter}}\sigma_{\text{intra}}}\right)\\
&= 1-\Phi\left(\frac{\delta_{\text{inter}}-2\epsilon}{2\sigma_{\text{intra}}}\right).
\end{aligned} \label{eq:probability_result}
\end{align}
\end{proof}

\begin{figure}[t]
\centering
\includegraphics[width=0.9\columnwidth]{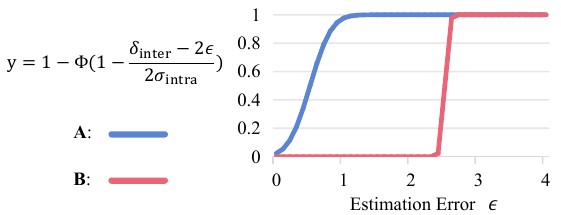} %
\caption{Illustration of an example to show the impact of estimation error on the probability of inter-class overlap. 
In this figure, $\delta_{\text{inter}} = 1$, $\sigma_{\text{intra}} = 0.5$ for {\bf A}, and $\delta_{\text{inter}} = 5$, $\sigma_{\text{intra}} = 0.1$ for {\bf B}, i.e., {\bf A} has a smaller $\delta_{\text{inter}}$ and a larger $\sigma_{\text{intra}}$ than {\bf B}, making it more susceptible to inter-class overlap caused by estimation errors.}
\label{fig:over}
\end{figure}

According to Lemma \ref{lem:overlap}, for complex fine-grained class distributions where $\delta_{\mathrm{inter}}$ is relatively small and $\sigma_{\mathrm{intra}}$ is relatively large, estimation errors are more likely to cause inter-class overlap, as illustrated in Fig. \ref{fig:over}.
Therefore, the accuracy of class distribution estimation is crucial for fine-grained classification.
Furthermore, Cramér-Rao Lower Bound \cite{1993Fundamentals} indicates that utilizing more data can reduce the class distribution estimation error.

Combining the above two points, it can be seen that the restrictions of existing learning paradigms on usable data become the key barrier limiting fine-grained classification.
Therefore, considering the real-world scenarios of fine-grained classification discussed in the introduction, we propose a novel learning paradigm to leverage all accessible off-the-shelf knowledge during system operation. 
Unlike traditional static class distribution estimation, this novel learning paradigm enables dynamic estimation based on data changes by allowing learning during inference.

Based on the above analysis and the proposed learning paradigm, we design a novel strategy to leverage unlabeled data $\mathcal{D}_{test}^{t}$ encountered during inference.
This strategy is implemented as a plug-in method built upon the FSCIL framework.
In light of the high efficiency demands during inference, the proposed method, consistent with prior approaches, employs class prototypes as concise estimates of class distributions.
The prototype is the mean vector of all labeled samples for each class $c$ and serves as the classifier weight during inference,
\begin{equation}
\mathbf{w}_{c}=\frac{1}{Num_c}\sum_{y_{i}=c}f(\mathbf{x}_{i}), 
\end{equation}
where $Num_c$ denotes the number of labeled samples in class $c$, and $f(\mathbf{x}_{i})$ represents the feature vector of sample $\mathbf{x}_{i}$.
The test sample label is predicted by finding the nearest prototype classifier $\mathbf{w}_{c}$.
Fig. \ref{fig1} illustrates the differences between prior arts and our EXP2 method. Algorithm \ref{alg:overall} outlines the overall method process, consisting of exploring useful knowledge (Sec. \ref{exploring}) and exploiting knowledge to optimize classifiers (Sec. \ref{exploiting}).



\subsection{Exploring Useful Knowledge}
\label{exploring}
During inference, the system will encounter a lot of unlabeled data provided by users, which are off-the-shelf and cost-free. Therefore, we attempt to explore useful knowledge for all seen classes from this data.


Specifically, 
test samples $\mathcal{D}_{test}^{t}$ will arrive during the inference process of session $t$, and we need to explore test samples that fit each class distribution.
For each class representation $\mathbf{w}_c$, the cosine similarity between $\mathbf{w}_c$ and feature vectors of all accessible test samples $\mathbf{x}$ are calculated as follows, 
\begin{equation}
\mathcal{S}(\mathbf{w}_{c},f(\mathbf{x}))=\frac{\mathbf{w}_{c}^{\top} f(\mathbf{x})}{\left \| \mathbf{w}_{c} \right \| \cdot \left \| f(\mathbf{x}) \right \|},  \mathbf{x} \in \mathcal{D}_{test}^{t}.
\label{eq:1}
\end{equation}

First, we select the samples corresponding to $R$ feature vectors with the highest similarity to $\mathbf{w}_{c}$,
\begin{equation}
\mathcal{X}_c=\{\mathbf{x}|\mathbf{x} \in \mathcal{D}_{test}^{t}, \mathcal{S}(\mathbf{w}_{c},f(\mathbf{x}))\ge \alpha _c\},
\label{eq:2}
\end{equation}
where $\alpha _c$ is the $R$-th highest cosine similarity between $\mathbf{w}_c$ and feature vectors $f(\mathbf{x})$ ($\mathbf{x} \in \mathcal{D}_{test}^{t}$).
Setting the same $R$ for all classes can help maintain balance among these classes to some extent.
Then, considering the small inter-class differences among fine-grained categories, samples with low similarity to $\mathbf{w}_c$ are filtered out using a confidence threshold $\tau $, as follows, 
\begin{equation}
\mathcal{X}_c=\{\mathbf{x}|\mathbf{x} \in \mathcal{X}_c, \mathcal{S}(\mathbf{w}_{c},f(\mathbf{x}))> \tau \}.
\label{eq:3}
\end{equation}
This step further ensures the effectiveness of the explored knowledge, 
and the threshold setting allows this method to function properly in an open environment, meaning it could reject samples from unseen classes.

This process yields the exploration boundary of each class with respect to the test samples, and samples $\mathcal{X}_c$ within the boundaries will be used to optimize the prototype of class $c$.

\begin{figure}[!t]
\begin{algorithm}[H]
\caption{Inference process in session $t$.}
\begin{algorithmic}[1] 
\REQUIRE The batch of test samples $\mathcal{D}_{test}^{t}$, feature extractor $f(\cdot )$, and linear classifiers $\mathbf{W}=\{\mathbf{w}_c| c\in \mathcal{C}^{0}\cdots\cup\mathcal{C}^t\}$.
\ENSURE The prediction results for all $\mathbf{x}\in \mathcal{D}_{test}^{t}$.
\STATE Compute the feature vector $f(\mathbf{x})$ for $\mathbf{x}\in \mathcal{D}_{test}^{t}$
\FOR{$\mathbf{w}_c \in \mathbf{W}$} 
\STATE // Exploring useful knowledge
\STATE $\mathcal{S}$$\gets$ Compute the cosine similarity between $\mathbf{w}_c$ and  $f(\mathbf{x}) (\mathbf{x} \in \mathcal{D}_{test}^{t})$ (Eq. \eqref{eq:1})
\STATE $\mathcal{X}_c$$\gets$ Select and filter out samples from $\mathcal{D}_{test}^{t}$ according to $\mathcal{S}$ (Eq. \eqref{eq:2} and Eq. \eqref{eq:3})
\STATE // Exploiting knowledge to optimize classifiers
\IF {$c \in \mathcal{C}^{0}$}
\STATE $\beta_c \gets \beta^{t+1}$ // Compute update degree
\ELSE
\STATE $\beta_c \gets \beta^{t-\left \lfloor \frac{c-|\mathcal{C}^{0}|}{N}  \right \rfloor }$ // Compute update degree
\ENDIF
\STATE $\mathbf{w}_c$$\gets$ Update $\mathbf{w}_c$ using $\mathcal{X}_c$ and $\beta_c$ (Eq. \eqref{eq:4})
\ENDFOR
\STATE \textbf{return} $y^\star=\underset{c\in {\textstyle \cup_{j=0}^{t}} \mathcal{C}^{j}}{\mathrm{argmin}} {\left \|  f(\mathbf{x}) -\mathbf{w}_{c}\right \|}^2$ for all $\mathbf{x} \in \mathcal{D}_{test}^{t}$
\end{algorithmic}
\label{alg:overall}
\end{algorithm}
\end{figure}

\subsection{Exploiting Knowledge to Optimize Classifiers}
\label{exploiting}
Existing FSCIL methods derive classifiers from scarce labeled samples through direct computation or backpropagation (BP), which is insufficient to fully describe a fine-grained class and often results in poor performance for incremental classes.
Thus, we optimize prototype classifiers (i.e., class representations) using the explored knowledge.


Specifically,
in the inference stage of session $t$, the classifier is updated by taking the weighted average of the original classifier $\mathbf{w}_c$ and the feature vectors of the explored unlabeled samples $\mathcal{X}_c$,
\begin{equation}
  \mathbf{w}_c  = (1-\beta_c) \cdot \mathbf{w}_c +\beta_c \cdot \frac{1}{|\mathcal{X}_c|}\sum_{\mathbf{x} \in \mathcal{X}_c} f(\mathbf{x}),
  \label{eq:4}
\end{equation}
where $\beta_c$ controls the degree of updating the classifier for class $c$.
Classifiers that do not encounter useful knowledge during exploration will not be updated.
As the classifier incorporates more and more knowledge, it gradually approaches the mean of the real class distribution. 
Therefore, during the incremental process, $\beta_c$ needs to be continuously reduced from its initial value, thereby reducing the risk of crossing the subtle inter-class gaps in fine-grained categories.
The specific calculation way is as follows,
\begin{equation} 
\beta_c=\begin{cases}
\beta^{t+1}, c \in \mathcal{C}^{0}
 \\
\beta^{t-\left \lfloor \frac{c-|\mathcal{C}^{0}|}{N}  \right \rfloor }, c \in \mathcal{C}^{1}\cdots\cup\mathcal{C}^t
\end{cases},
\label{eq:5}
\end{equation}
where $\beta$ is a hyperparameter in the range $(0, 1)$. 
As session $t$ increases, $\beta_c$ decays exponentially.
Thus, the earlier a class starts being learned, the smaller the degree of updates to its classifier.

This process not only aligns the classifiers more closely with the real data distribution but also allows them to reflect data changes in real-time.
Since the open-set performance of a classifier can be improved by enhancing its closed-set accuracy \cite{DBLP:conf/iclr/Vaze0VZ22}, the explored knowledge also becomes increasingly accurate.
Through this mutually reinforcing process, the classifier would gradually approach the optimal from a holistic perspective.

\begin{table}[t]
\centering
\caption{Statistics of datasets.}
\begin{tabular}{lccccc}
    \toprule
    \multicolumn{1}{c}{Dataset} & Total Classes &   $|\mathcal{C}^{0}|$    & T     & N     & K \\
    \midrule
    CUB200 & 200 & 100   & 10    & 10    & 5 \\
    FGVCAircraft & 100 & 50    & 10    & 5     & 5 \\
    StanfordCars & 196 & 96    & 10    & 10    & 5 \\
    StanfordDogs & 120  & 70    & 10    & 5     & 5 \\
    CIFAR100 & 100 & 60    & 8     & 5     & 5 \\
    {\em mini}ImageNet & 100 & 60    & 8     & 5     & 5 \\
    \bottomrule
    \end{tabular}%

\label{tab:data}
\end{table}

\section{Experiments}

\subsection{Experimental Setup}
\label{sec:set}
\subsubsection{Datasets}
In addition to the fine-grained dataset CUB200 \cite{Wah2011TheCB} commonly used in the FSCIL problem, we introduce three new fine-grained datasets: 
FGVCAircraft \cite{DBLP:journals/corr/MajiRKBV13}, StanfordCars \cite{DBLP:conf/cvpr/YangLLT15}, and StanfordDogs \cite{Khosla2012NovelDF}.
This allows for a more comprehensive demonstration of the method's performance across various types of fine-grained datasets.
Additionally, we conduct experiments on coarse-grained datasets CIFAR100 \cite{Krizhevsky2009LearningML} and {\em mini}ImageNet \cite{DBLP:journals/ijcv/RussakovskyDSKS15} commonly used in FSCIL to demonstrate the method's versatility.
For three commonly used datasets in FSCIL, we follow the split in previous works \cite{DBLP:conf/cvpr/ZhangSLZPX21, DBLP:journals/pami/YangLZLLJY23, DBLP:conf/cvpr/SongZSP0023}. 
For three newly introduced datasets, we split them according to existing ways.
The statistics of datasets are listed in Table \ref{tab:data}.

\begin{table*}[!t]
  \centering
  \caption{The overall accuracy (Overall) and incremental class accuracy (Inc.) of sessions 2, 4, 6, 8, and 10 on the CUB200 dataset.}
   \begin{tabular}{lcccccccccc}
    \toprule
    \multicolumn{1}{c}{\multirow{2}[4]{*}{{\bf Method}}} &  \multicolumn{2}{c}{2} & \multicolumn{2}{c}{4} & \multicolumn{2}{c}{6} & \multicolumn{2}{c}{8} & \multicolumn{2}{c}{10 (Last)} \\
\cmidrule{2-11}       & Overall & Inc.  & Overall & Inc.  & Overall & Inc.  & Overall & Inc.  & Overall & Inc. \\
    \midrule
    
    FACT \cite{DBLP:conf/cvpr/0001WYMPZ22} & 70.84  & 47.02  & 65.56  & 39.70  & 61.74  & 39.25  & 58.41  & 38.25  & 56.94  & 39.74  \\
    CLOM \cite{DBLP:conf/nips/ZouZLL22}  & 72.94  & 37.44  & 67.80  & 36.49  & 63.94  & 37.97  & 60.62  & 38.69  & 59.58  & 40.67  \\
    Bidist \cite{DBLP:conf/cvpr/Zhao0XC0NF23}  & 70.12  & 51.14  & 64.37  & 44.28  & 60.71  & 41.93  & 57.41  & 40.83  & 55.94  & 42.15  \\
    OSHHG \cite{DBLP:journals/tmm/CuiYPTL24}  & 59.83  & 43.17  & 55.07  & 34.99  & 51.56  & 32.41  & 47.50  & 28.13  & 45.87  & 28.73  \\
    MICS \cite{DBLP:conf/wacv/KimJPY24}  & 72.25  & 46.29  & 67.43  & 44.24  & 64.61  & 45.66  & 61.77  & 44.66  & 61.29  & 47.47  \\
    M2SD \cite{DBLP:conf/aaai/LinWL0L24}  & 73.58  &    -   & 68.73  &  -  & 64.73  &    -   & 62.70  &   -    & 60.96  & - \\
    DyCR \cite{10531293}  & 71.69  & 42.76  & 66.59  & 41.64  & 62.66  & 42.01  & 60.57  & 39.51  & 58.46  & 41.90  \\
    ADBS \cite{DBLP:conf/aaai/LiTY0DY25}  & 73.51 & 50.74 & 68.59 & 45.54 & 64.36 & 45.35 & 61.84 & 44.20 & 60.14 & 45.75\\
    SAVC+D2A \cite{DBLP:conf/cvpr/Zhao000X25} & 75.28 & 55.15 &  69.71 & 48.07 & 65.37 &  46.48 &  62.49 &  45.52 &  61.46 &  47.43 \\
    \midrule
    Decoupled-Cosine \cite{DBLP:conf/nips/VinyalsBLKW16} & 67.41  & 34.31  & 60.30  & 29.09  & 54.94  & 26.74  & 51.04  & 25.77  & 48.81  & 26.99  \\
    +T3A \cite{DBLP:conf/nips/IwasawaM21}  & 66.65  & 32.71  & 59.53  & 27.87  & 54.83  & 27.34  & 51.33  & 27.37  & 48.83  & 27.69  \\
    +TAST \cite{DBLP:conf/iclr/JangCC23}  & 61.43  & 0.00  & 52.01  & 0.00  & 45.71  & 0.00  & 40.30  & 0.00  & 36.45  & 0.13  \\
    \rowcolor{blue!10}\textit{+EXP2 (Ours) }  & 67.26  & 35.38  & 60.30  & 30.64  & 55.60  & 29.31  & 51.54  & 28.04  & 49.33  & 28.94  \\
    \midrule
    CEC \cite{DBLP:conf/cvpr/ZhangSLZPX21}   & 68.48  & 37.45  & 62.21  & 33.82  & 57.67  & 32.75  & 53.94  & 31.46  & 52.21  & 33.02  \\
    +T3A   & 67.90  & 34.81  & 61.72  & 31.92  & 56.81  & 31.01  & 51.46  & 27.79  & 48.95  & 29.18  \\
    +TAST  & 62.19  & 0.00  & 53.08  & 0.00  & 46.03  & 0.00  & 41.07  & 1.01  & 36.78  & 0.81  \\
    \rowcolor{blue!10}\textit{+EXP2 (Ours) }  & 69.39  & 43.87  & 63.26  & 38.36  & 58.43  & 35.77  & 54.46  & 33.68  & 53.00  & 35.20  \\
    \midrule
    TEEN \cite{DBLP:conf/nips/Wang0ZZY23}  & 72.15  & 52.56  & 66.88  & 47.41  & 62.76  & 46.10  & 60.04  & 43.54  & 58.48  & 44.84  \\
    +T3A   & 71.80  & 50.62  & 65.84  & 43.56  & 63.17  & 44.08  & 59.43  & 43.12  & 58.44  & 45.40  \\
    +TAST  & 53.24  & 0.00  & 45.73  & 0.00  & 39.42  & 0.00  & 34.61  & 0.00  & 31.12  & 0.00  \\
    \rowcolor{blue!10}\textit{+EXP2 (Ours) }  & 72.52  & \textbf{55.68} & 67.87  & 50.06  & 64.17  & \textbf{49.79} & 61.82  & 48.18  & 60.49  & 49.23  \\
    \midrule
    SAVC \cite{DBLP:conf/cvpr/SongZSP0023}  & 75.42  & 52.09  & 70.23  & 47.15  & 66.15  & 45.41  & 63.56  & 45.49  & 62.67  & 47.55  \\
    +T3A   & 74.31  & 51.78  & 68.45  & 45.41  & 64.26  & 43.87  & 61.06  & 43.94  & 60.08  & 45.62  \\
    +TAST  & 57.84  & 3.17  & 50.40  & 5.33  & 43.84  & 3.55  & 38.99  & 4.22  & 35.54  & 3.46  \\
    \rowcolor{blue!10}\textit{+EXP2 (Ours) }&   \textbf{75.60} & 55.48  & \textbf{70.61} & \textbf{50.21} & \textbf{66.67} & 49.37  & \textbf{63.95} & \textbf{48.96} & \textbf{63.58} & \textbf{50.77} \\
    \midrule
    PFR \cite{DBLP:journals/pr/PanZYZG24}  & 71.40  & 43.28  & 67.96  & 39.51  & 62.87  & 40.83  & 59.59  & 40.90  & 57.54  & 42.90  \\
    +T3A  & 71.27  & 46.91  & 66.16  & 41.98  & 63.03  & 42.76  & 60.27  & 42.04  & 58.52  & 43.76  \\
    +TAST  & 70.44  & 41.22  & 64.15  & 37.75  & 60.81  & 38.96  & 58.32  & 38.95  & 57.25  & 41.42  \\
    \rowcolor{blue!10}\textit{+EXP2 (Ours) }  & 71.70  & 48.29  & 66.72  & 43.79  & 63.80  & 45.21  & 61.68  & 44.57  & 60.11  & 46.55  \\
    \bottomrule
    \end{tabular}%
    
  \label{tab:com}%
\end{table*}%

\subsubsection{Experimental details} 
As discussed earlier, our method is independent of the training phase and the baseline model. 
Training-related settings follow the baseline methods, and results are reproduced using public code in the same environment.
The experiments are conducted using PyTorch library on a single NVIDIA 3090.
To ensure experimental fairness and rationality, we conduct experiments based on representative FSCIL methods, including Decoupled-Cosine, CEC, TEEN, SAVC, and PFR.
Decoupled-Cosine and CEC serve as classic baselines.
TEEN and SAVC are strong methods that emphasize improving performance on incremental classes with limited labeled data. PFR is a rare FSCIL method specifically designed for fine-grained datasets.
In addition, we compare our method with two plug-in TTA methods capable of utilizing test data: T3A, a rare TTA method that does not rely on BP for optimization, and TAST, which represents the typical approach that utilizes BP-based parameter optimization.
For the newly introduced fine-grained datasets, we set hyperparameters to be consistent with the CUB200 dataset. 
The $R$ and $\tau $ are consistently set to 40 and 0.8 across all datasets and baselines. 
Due to the base classes having more training samples compared to the incremental classes, we set $\beta$ to 0.05 for base classes and 0.3 for incremental classes across all datasets and baselines.
These hyperparameter settings are chosen as they achieve the best average performance across datasets and baselines.
The source code of the proposed method is available at
\url{https://github.com/Legenddddd/EXP2}.

\begin{figure*}[t]
\centering
\includegraphics[width=1\textwidth]{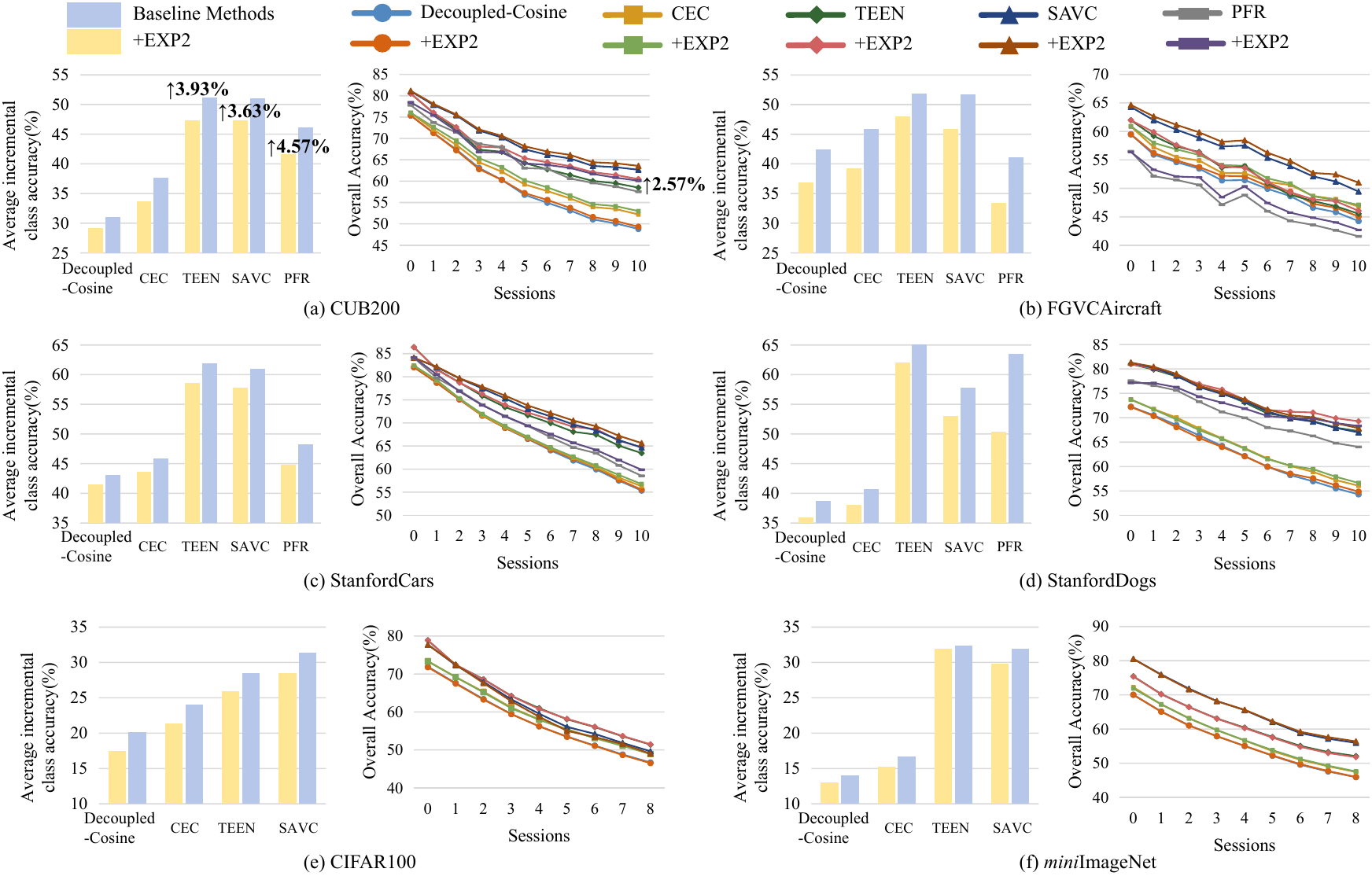} 
\caption{The average of incremental class accuracy across different sessions and the overall accuracy of each session.
}
\label{fig:com}
\end{figure*}

\subsection{Comparison Results}
In this section, we compare our method with recent representative FSCIL methods and plug-and-play TTA methods in Table \ref{tab:com} and Fig. \ref{fig:com}.
Given that incremental classes have limited labeled data and are typically poorly represented by the original methods, we separately report the incremental class accuracy.


\subsubsection{EXP2 consistently and significantly improves the performance of FSCIL methods, especially for fine-grained data and incremental classes}
(1) It can be observed that although FSCIL methods such as TEEN and SAVC have focused on the accuracy of incremental classes and already perform well in FSCIL, their incremental class accuracy remains significantly lower than overall accuracy on fine-grained datasets (Table \ref{tab:com}). 
Applying our proposed method to these approaches helps better characterize fine-grained class representations, further improving incremental class accuracy, boosting overall accuracy, and achieving state-of-the-art results (Table \ref{tab:com} and Fig. \ref{fig:com}).
(2) Furthermore, our method is more effective on fine-grained datasets compared to traditional ones (Fig. \ref{fig:com}). 
This validates our discussion that fine-grained datasets, with large intra-class variance and small inter-class differences, make accurate class representation more challenging and increase the risk of negatively impacting classification results.
(3) Applying our method to PFR, which focuses on discriminative feature extraction for fine-grained categories, achieves consistently more remarkable performance improvements (Fig. \ref{fig:com}). This indicates that our method can mutually complement and enhance existing fine-grained classification approaches that emphasize feature extraction.

In summary, our method shows significant potential when applied to already well-performing FSCIL approaches and fine-grained feature extraction methods, enhancing fine-grained classification performance by supplementing and refining category representations.

\subsubsection{ EXP2 outperforms TTA methods in fine-grained classification with complex and dynamically changing semantic distributions}
(1) T3A is a rare TTA method that also does not use BP optimization. However, its overemphasis on test data likely introduces noise, causing it to fail on most baselines (Table \ref{tab:com}).
(2) TAST represents the common approach of TTA methods, utilizing BP for parameter optimization. 
While effective BP optimization can work with few test samples and epochs in simple domain generalization tasks, it collapses on complex and dynamically changing fine-grained distributions, resulting in poor performance (Table \ref{tab:com}).
Further exploration of this interesting phenomenon can be found in Sec. \ref{sec:tta}.
(3) PFR focuses on fine-grained feature extraction while avoiding classifier and feature extractor optimization during the incremental phase, ensuring a stable class distribution and thus resisting severe TTA-induced degradation.

In summary, unlike TTA which also utilizes test data, our method can not only adapt to dynamic data changes but also effectively optimize the complex distribution of fine-grained categories.
Thus, the proposed learning paradigm and strategy hold unique research value and significance.


\subsection{Analyses}
In this section, we conduct a thorough analysis of the EXP2 strategy and method on the Decoupled-Cosine baseline.

\begin{table*}[!t]
  \centering
  \caption{Comparison of performance in different ways of exploiting knowledge on the CUB200 dataset.}
    \begin{tabular}{lccccccccccc}
    \toprule
    \multicolumn{1}{c}{\multirow{2}[4]{*}{\makecell{Exploiting\\Way}}} & \multicolumn{2}{c}{2} & \multicolumn{2}{c}{4} & \multicolumn{2}{c}{6} & \multicolumn{2}{c}{8} & \multicolumn{2}{c}{10 (Last)} & \multicolumn{1}{c}{\multirow{2}[4]{*}{\makecell{Inference\\Time (ms)}}} \\
\cmidrule{2-11}          & Overall & Inc.  & Overall & Inc.  & Overall & Inc.  & Overall & Inc.  & Overall & Inc. \\
    \midrule
    Baseline & 67.41  & 34.31  & 60.30  & 29.09  & 54.94  & 26.74  & 51.04  & 25.77  & 48.81  & 26.99 & 1.241\\
    \textbf{Average} & 63.27  & 23.67  & 53.38  & 18.14  & 45.15  & 16.47  & 39.03  & 14.24  & 36.45  & 16.21 & 1.253\\
    \textbf{Weight} & \textbf{67.46} & 33.94  & \textbf{60.40} & 28.48  & 55.20  & 26.00  & 51.15  & 25.10  & 48.93  & 26.20 & 1.244\\
    \textbf{Finetuning} & 67.43  & 33.77  & 60.25  & 28.57  & 54.92  & 26.23  & 50.98  & 25.30  & 48.81  & 26.61 & 1.302\\
    Ours  & 67.23  & \textbf{35.38} & 60.30  & \textbf{30.64} & \textbf{55.60} & \textbf{29.31} & \textbf{51.61} & \textbf{28.04} & \textbf{49.33} & \textbf{28.94} & 1.245\\
    \bottomrule
    \end{tabular}%
    
  \label{tab:ana}%
\end{table*}%

\subsubsection{The basic strategy is practical}  Fig. \ref{fig:abl} illustrates the impact of the hyperparameters $R$ and $\tau $ (Eq. \eqref{eq:2} and Eq. \eqref{eq:3}) on the performance in the last session. Where $\tau =1$ denotes filtering out all samples, representing the original performance of the baseline. 
Our method demonstrates universal performance improvements across different values and achieves superior performance under specific hyperparameter settings.
This indicates that with an appropriate exploiting knowledge approach, the model only needs to identify samples that better reflect the class distribution compared to others, without having to meet the stringent requirements that are often impractical in practice.
Consequently, leveraging cost-free data encountered during system operation is a feasible strategy in practical scenarios, and our proposed Exploring Useful Knowledge process further ensures the quality of the utilized knowledge, thereby avoiding adverse effects such as noise.

\subsubsection{The way of exploiting useful knowledge is crucial}
Table \ref{tab:ana} demonstrates several different ways of exploiting knowledge. It can be observed that simply averaging the classifier and the feature vectors of all test samples in each session (\textbf{Average}) significantly reduces both overall accuracy and incremental class accuracy. 
This indicates that controlling the degree of the influence of the explored knowledge is crucial.  
While employing weighted averaging (\textbf{Weight}, Eq. \eqref{eq:4}) tends to maintain overall accuracy, incremental class accuracy still declines. 
Further decaying the weights in the weighted averaging based on the session (\textbf{Ours}, Eq. \eqref{eq:4} and Eq. \eqref{eq:5}) can improve overall accuracy and incremental class accuracy. 
This implies that for fine-grained classes with small inter-class differences, continuous equal-step movement could cause the class representations to deviate from the correct position. 
Besides, directly using BP to optimize the classifier (\textbf{Finetuning}) leads to a slight performance drop in the complex and unstable fine-grained class distributions. Thus, compared to other methods, our method is a more reasonable way to leverage knowledge.

\begin{figure}[t]
\centering
\includegraphics[width=0.95\columnwidth]{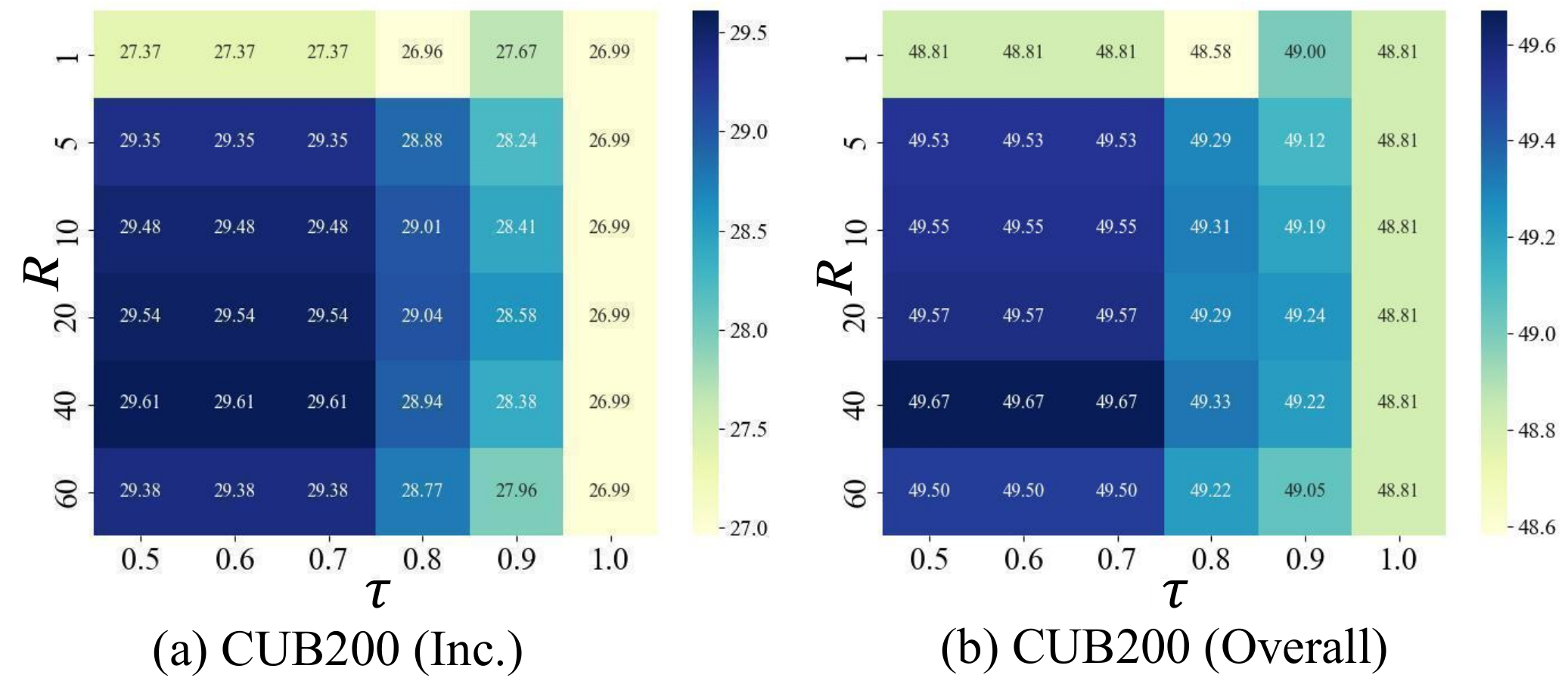} 
\caption{The impact of hyperparameters $R$ and $\tau$ on performance.}
\label{fig:abl}
\end{figure}

\subsubsection{The strategy is more effective when labeled training samples are insufficient}
Fig. \ref{fig:abl1} illustrates the impact of the hyperparameter $\beta$ (Eq. \eqref{eq:5}) in the last session on the CUB200 dataset, where the number of base classes and incremental classes is equal.
We fix $\beta$ for base (incremental) classes and attempt to record the impact of different $\beta$ for incremental (base) classes on performance. 
It can be observed that when $\beta > 0$, better results can be achieved. Moreover, the performance improvement for incremental classes is significantly greater than that for base classes.
This suggests that the strategy is more effective for incremental classes with scarce training samples. 
Since the base class classifiers are already positioned stably near the center of their fine-grained class distributions, the process of exploring and exploiting knowledge has a limited impact on their performance.

\begin{figure}[t]
\centering
\includegraphics[width=1\columnwidth]{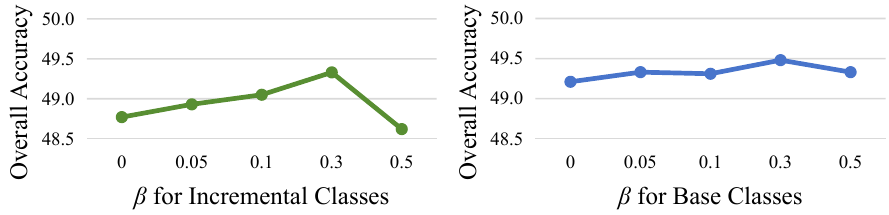} 
\caption{The impact of hyperparameter $\beta $ on performance.}
\label{fig:abl1}
\end{figure}

\subsubsection{The EXP2 method is lightweight and efficient} 
As an inference-stage approach,
it does not increase the model's parameters or FLOPs. The last column of Table \ref{tab:ana} shows the average inference time per sample during the incremental process under the same experimental environment.
The time of {\bf Finetuning} refers to optimizing the classifier for a single epoch (serving as a lower bound for finetuning-based methods), yet it is still significantly slower than our method.
This suggests that using the explored knowledge to optimize more components, such as the entire model, would further reduce efficiency.
Our Exploring and Exploiting process involves only a few simple computations and, like other non-backpropagation methods (\textbf{Average} and \textbf{Weight}), has negligible impact on inference time.
This demonstrates the efficiency of our forward-only classifier optimization.

\begin{figure}[t]
\centering
\includegraphics[width=0.95\columnwidth]{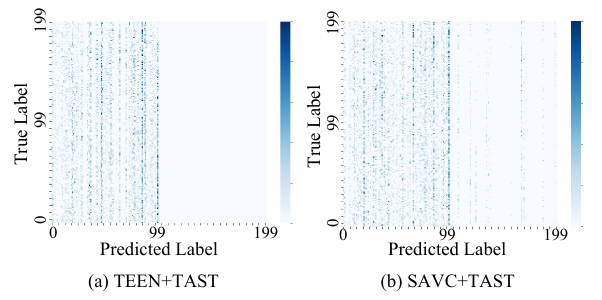} 
\caption{Confusion matrix of the classification results in the final session. Labels 0–99 indicate base classes, while labels 100–199 correspond to incremental classes.}
\label{fig:abl2}
\end{figure}

\subsubsection{Conventional TTA methods often lead to classification collapse for classes with few labeled samples}
\label{sec:tta}
Fig.~\ref{fig:abl2} illustrates the classification collapse phenomenon caused by backpropagation-based TTA methods, such as TAST, on the CUB200 dataset. It can be observed that most misclassified incremental class samples are incorrectly predicted as base class samples. Similar observations have also been reported in prior FSCIL studies~\cite{DBLP:conf/nips/Wang0ZZY23}, while conventional TTA methods often further aggravate this phenomenon. 
This is because FSCIL methods that do not focus on fine-grained feature extraction (e.g., TEEN and SAVC) tend to yield more unstable class representations for incremental classes than PFR, which is tailored for fine-grained datasets. 
As a result, applying conventional TTA methods not only fails to refine these representations using test data, but also distorts them, thereby exacerbating the inherent bias toward base classes in the FSCIL paradigm. 
This highlights the difficulty of optimizing class representations for fine-grained semantic distributions that are complex and dynamically changing.

\section{Conclusion and Future Work}
In this paper, we analyze the barriers faced by fine-grained classification in real-world applications and propose a learning paradigm and strategy to break these barriers using cost-free data. 
On one hand, as an inference-stage method for fine-grained classification, this strategy does not conflict with existing advanced fine-grained feature extraction networks, enabling further performance improvements. 
On the other hand, when applied within the FSCIL framework, 
this strategy could inspire a shift in FSCIL method design towards a more dynamic and open perspective, offering deeper insights and exploration of the FSCIL scenario.

For experimental fairness and clarity, this paper is based on the FSCIL scenario, providing preliminary evidence for the research value of the novel learning paradigm.
As a plug-in method, it can be applied to more flexible and open scenarios, including those without category restrictions on test data.
In future work, we will extend this new problem and learning paradigm beyond the FSCIL framework, developing more refined and diverse settings.



\bibliographystyle{IEEEtran}
\bibliography{egbib}


\vfill

\end{document}